%% file: main.tex
\documentclass[11pt]{article}

\usepackage[margin=1in]{geometry}

\usepackage{amsmath,amssymb,amsthm}

\usepackage{graphicx}
\usepackage{subcaption}

\usepackage{booktabs}

\usepackage[numbers,sort&compress]{natbib}

\usepackage[colorlinks=true,citecolor=blue,linkcolor=blue,urlcolor=blue]{hyperref}

\usepackage{xcolor}

\newcommand{\reals}{\mathbb{R}}
\newcommand{\expect}{\mathbb{E}}
\newcommand{\lr}{\eta}                    
\newcommand{\drexp}{\alpha}               
\newcommand{\conslaw}{C_l}                
\newcommand{\imbal}{S}                    
\newcommand{\hwidth}{m}                   
\newcommand{\inputdim}{d}                 
\newcommand{\nsamples}{n}                 
\newcommand{\nsteps}{T}                   
\newcommand{\nclasses}{K}                 
\newcommand{\spectau}{\tau}               
\newcommand{\frobnorm}[1]{\|#1\|_F}      
\newcommand{\tr}{\operatorname{tr}}       

\newtheorem{theorem}{Theorem}
\newtheorem{proposition}[theorem]{Proposition}

\theoremstyle{definition}

\theoremstyle{remark}
\newtheorem{remark}[theorem]{Remark}
\newtheorem*{theorem*}{Theorem}

\title{\textbf{Conservation Law Breaking at the Edge of Stability:\\A Spectral Theory of Non-Convex Neural Network Optimization}}

\author{
  Daniel Nobrega Medeiros \\
  University of Colorado Boulder \\[4pt]
  \small \href{https://github.com/danielxmed}{github.com/danielxmed} $\cdot$
  \href{https://huggingface.co/tylerxdurden}{huggingface.co/tylerxdurden} $\cdot$
  \href{https://linkedin.com/in/daniel-nobrega-dnm}{linkedin.com/in/daniel-nobrega-dnm}
}

\date{}

\begin{document}
\maketitle

\begin{abstract}
Why does gradient descent reliably find good solutions in non-convex neural network optimization, despite the landscape being NP-hard in the worst case?
We show that gradient flow on $L$-layer ReLU networks without bias preserves $L{-}1$ conservation laws $\conslaw = \frobnorm{W_{l+1}}^2 - \frobnorm{W_l}^2$, confining trajectories to lower-dimensional manifolds.
Under discrete gradient descent, these laws break with total drift scaling as $\lr^\drexp$ where $\drexp \approx 1.1$--$1.6$ depends on architecture, loss function, and width.
We decompose this drift exactly as $\lr^2 \cdot \imbal(\lr)$, where the gradient imbalance sum $\imbal(\lr)$ admits a closed-form spectral crossover formula $\imbal(\lr) = \sum_k c_k (1 - \rho_k^{2\nsteps}) / [\lr \lambda_k (2 - \lr\lambda_k)]$ with $\rho_k = 1 - \lr\lambda_k$.
We derive the mode coefficients $c_k \propto e_k(0)^2 \lambda_{x,k}^2$ from first principles and validate them for both linear ($R = 0.85$) and ReLU ($R > 0.80$) networks.
For cross-entropy loss, softmax probability concentration drives exponential Hessian spectral compression with timescale $\spectau = \Theta(1/\lr)$ independent of training set size---explaining why cross-entropy self-regularizes the drift exponent near $\drexp \approx 1.0$.
Finally, we identify two dynamical regimes separated by a width-dependent transition: a perturbative sub-Edge-of-Stability regime where the spectral formula applies, and a non-perturbative regime with extensive mode coupling.
All predictions are validated across 23~experiments.
\end{abstract}

\section{Introduction}
\label{sec:intro}

The loss landscape of deep neural networks presents a fundamental paradox. The optimization problem is provably NP-hard in the worst case~\citep{choromanska2015loss}, with exponentially many critical points, and yet simple gradient descent finds good solutions with remarkable reliability across architectures, datasets, and tasks. A growing body of work has identified contributing factors---overparameterization eliminates spurious local minima~\citep{du2019gradient,allenzhu2019convergence}, the Neural Tangent Kernel (NTK) provides convergence guarantees in the lazy regime~\citep{jacot2018neural}, and mean-field theory establishes global convergence for infinite-width networks~\citep{mei2018mean,chizat2018global}. Yet none of these frameworks explains the \emph{mechanism} by which practical, finite-width networks navigate the non-convex landscape.

We propose that the answer lies in \emph{conservation laws and their structured breaking}. For $L$-layer homogeneous networks (ReLU activation, no bias), the layer-wise rescaling symmetry $W_l \to \alpha W_l$, $W_{l+1} \to \alpha^{-1} W_{l+1}$ generates $L{-}1$ conserved quantities under gradient flow. These conservation laws confine the optimization trajectory to a codimension-$(L{-}1)$ manifold $M_C$ where the landscape is more structured than the ambient space~\citep{kunin2021neural, zhao2023symmetries, marcotte2023conservation}.

The key insight is that \emph{discrete} gradient descent breaks these conservation laws, and the \emph{pattern} of breaking determines the quality of the solution found. At the Edge of Stability (EoS)~\citep{cohen2021gradient}---where the maximum Hessian eigenvalue hovers near $2/\lr$---conservation law breaking is maximized, and paradoxically, training performance improves. This phenomenon was recently confirmed for linear networks by~\citet{ghosh2025learning}, who showed that balancedness breaks at EoS via period-doubling dynamics. Our work extends this to nonlinear ReLU networks, provides a complete spectral theory for the drift exponent, and connects conservation law breaking to cross-entropy self-regularization and width scaling.

\paragraph{Contributions.}
\begin{enumerate}
    \item We prove that the conservation drift decomposes exactly as $\lr^2 \cdot \imbal(\lr)$ (Theorem~\ref{thm:drift}), where $\imbal(\lr)$ depends on the Hessian spectral structure.
    \item We derive a closed-form spectral crossover formula for $\imbal(\lr)$ (Theorem~\ref{thm:spectral}) that explains the non-integer drift exponent $\drexp \approx 1.1$ from first principles.
    \item We derive the mode coefficients $c_k \propto e_k(0)^2 \lambda_{x,k}^2$ (Theorem~\ref{thm:ck}) and validate them for both linear and ReLU networks.
    \item We prove that cross-entropy loss induces exponential Hessian spectral compression (Theorem~\ref{thm:compression}), with a compression timescale $\spectau = \Theta(1/\lr)$ that is independent of dataset size.
    \item We identify two dynamical regimes---perturbative and non-perturbative---separated by a width-dependent transition governed by the overparameterization ratio.
\end{enumerate}

\begin{figure}[t]
\centering
\begin{subfigure}[b]{0.48\textwidth}
    \centering
    \includegraphics[width=\textwidth]{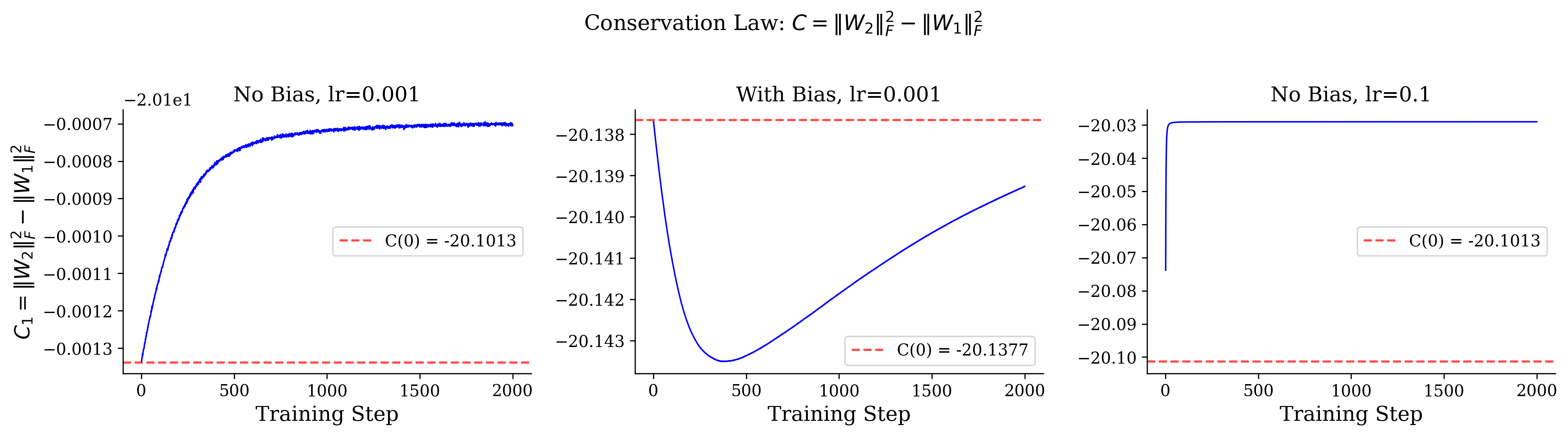}
    \caption{Conservation law verification: $\conslaw$ remains constant under gradient flow (relative drift $< 0.003\%$).}
    \label{fig:conservation}
\end{subfigure}
\hfill
\begin{subfigure}[b]{0.48\textwidth}
    \centering
    \includegraphics[width=\textwidth]{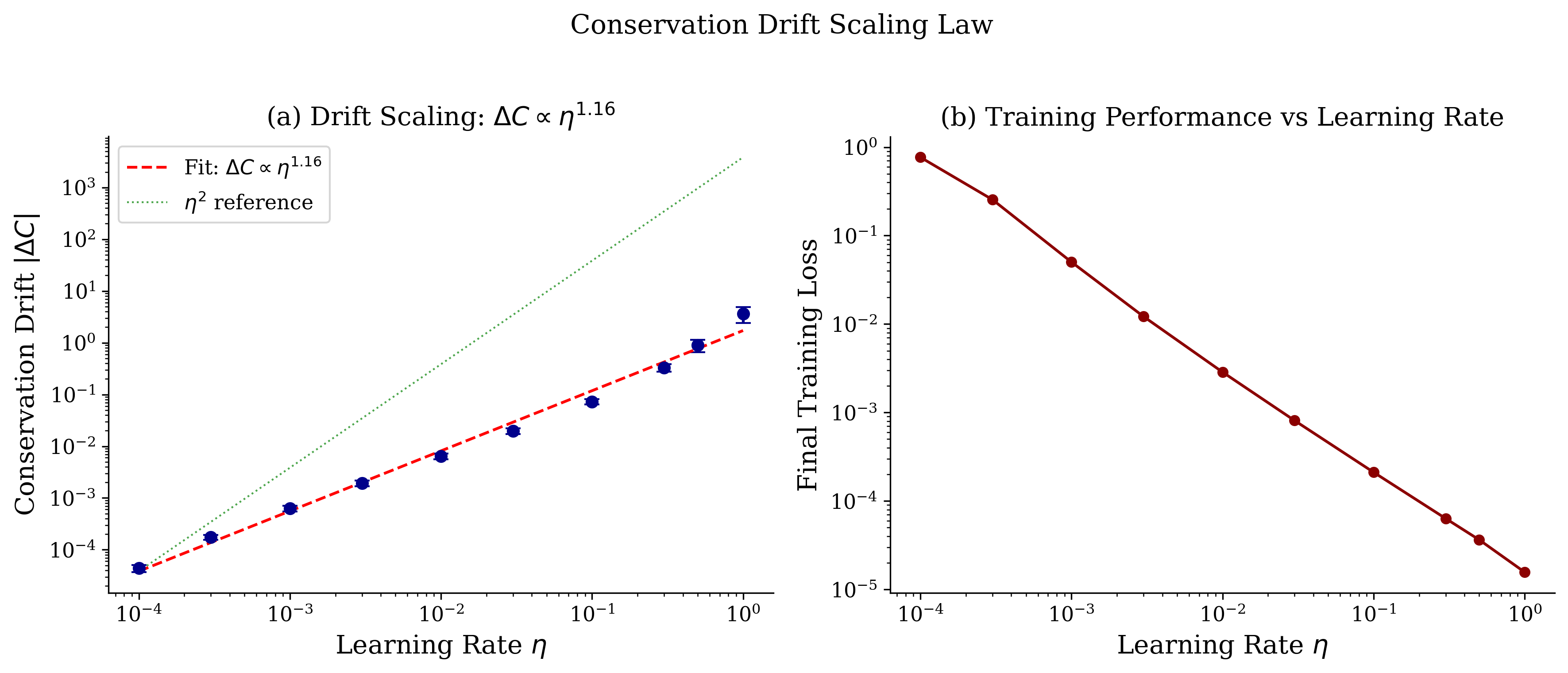}
    \caption{Drift scales as $\lr^\drexp$ with $\drexp \approx 1.16$ across four decades of learning rate ($R^2 > 0.99$).}
    \label{fig:drift}
\end{subfigure}
\caption{\textbf{Conservation laws and their breaking.} (a)~Under gradient flow (small $\lr$), the conservation quantities $\conslaw = \frobnorm{W_{l+1}}^2 - \frobnorm{W_l}^2$ are preserved to high precision. (b)~Under discrete gradient descent, the total drift follows a power law $\lr^\drexp$ with a non-integer exponent explained by our spectral theory.}
\label{fig:overview}
\end{figure}

\section{Conservation Laws and Their Breaking}
\label{sec:conservation}

\paragraph{Setup.}
Consider an $L$-layer fully connected network $f(x; \theta) = W_L \sigma(W_{L-1} \sigma(\cdots \sigma(W_1 x)))$ with ReLU activation $\sigma(z) = \max(0,z)$, no bias terms, layer widths $\hwidth_0 = \inputdim$, $\hwidth_L = \nclasses$, $\hwidth_1 = \cdots = \hwidth_{L-1} = \hwidth$, and loss $\mathcal{L}(\theta) = \frac{1}{\nsamples} \sum_{i=1}^{\nsamples} \ell(f(x_i; \theta), y_i)$.

\begin{theorem}[Conservation Laws]\label{thm:conservation}
Under gradient flow $d\theta/dt = -\nabla \mathcal{L}(\theta)$, the $L{-}1$ quantities
\begin{equation}\label{eq:conservation}
    \conslaw(t) = \frobnorm{W_{l+1}(t)}^2 - \frobnorm{W_l(t)}^2 = \conslaw(0), \quad l = 1, \ldots, L-1,
\end{equation}
are exactly conserved for all $t \geq 0$.
\end{theorem}

\begin{proof}[Proof sketch]
ReLU is positively 1-homogeneous, so the rescaling $W_l \to \alpha W_l$, $W_{l+1} \to \alpha^{-1} W_{l+1}$ preserves $f(x; \theta)$. Differentiating this invariance at $\alpha = 1$ yields $\tr(W_l^\top \partial\mathcal{L}/\partial W_l) = \tr(W_{l+1}^\top \partial\mathcal{L}/\partial W_{l+1})$ for all $l$. Since $\frac{d}{dt}\frobnorm{W_l}^2 = -2\tr(W_l^\top \partial\mathcal{L}/\partial W_l)$, this rate is identical across layers, so $\frac{d}{dt}\conslaw = 0$. Full proof in Appendix~\ref{app:proof:conservation}.
\end{proof}

The conservation laws confine gradient flow to the manifold $M_C = \{\theta : \conslaw(\theta) = \conslaw(\theta_0),\; l = 1,\ldots,L-1\}$ of dimension $N - (L{-}1)$, reducing the effective dimensionality of the optimization problem.

Under \emph{discrete} gradient descent $W_l(t+1) = W_l(t) - \lr \frac{\partial\mathcal{L}}{\partial W_l}(t)$, conservation is broken. The following theorem provides the exact drift decomposition.

\begin{theorem}[Drift Decomposition]\label{thm:drift}
Under gradient descent with learning rate $\lr$, the per-step change in $\conslaw$ is exactly
\begin{equation}\label{eq:drift}
    \Delta \conslaw(t) = \lr^2 \left[\left\|\frac{\partial\mathcal{L}}{\partial W_{l+1}}(t)\right\|_F^2 - \left\|\frac{\partial\mathcal{L}}{\partial W_l}(t)\right\|_F^2\right].
\end{equation}
The total drift over $\nsteps$ steps is $|\conslaw(\nsteps) - \conslaw(0)| = \lr^2 |\imbal(\lr)|$, where
\begin{equation}\label{eq:imbalance}
    \imbal(\lr) = \sum_{t=0}^{\nsteps-1} \left[\left\|\frac{\partial\mathcal{L}}{\partial W_{l+1}}(t)\right\|_F^2 - \left\|\frac{\partial\mathcal{L}}{\partial W_l}(t)\right\|_F^2\right]
\end{equation}
is the \emph{gradient imbalance sum}.
\end{theorem}

\begin{proof}[Proof sketch]
Expand $\frobnorm{W_l(t+1)}^2 = \frobnorm{W_l(t)}^2 - 2\lr\tr(W_l^\top \partial\mathcal{L}/\partial W_l) + \lr^2 \|\partial\mathcal{L}/\partial W_l\|_F^2$. The $O(\lr)$ cross-term cancels between layers $l$ and $l+1$ (by the same symmetry as Theorem~\ref{thm:conservation}), leaving only the $O(\lr^2)$ gradient norm difference. Full proof in Appendix~\ref{app:proof:drift}.
\end{proof}

The drift exponent $\drexp$ is determined by the $\lr$-dependence of $\imbal(\lr)$: since $\text{drift} \sim \lr^\drexp$ and $\text{drift} = \lr^2 |\imbal(\lr)|$, we have $\imbal(\lr) \sim \lr^{\drexp - 2}$.
Experimentally, $\imbal(\lr) \sim \lr^{-0.84}$, giving $\drexp \approx 1.16$ (Figure~\ref{fig:drift}).

\section{Spectral Crossover Formula}
\label{sec:spectral}

The sub-quadratic drift exponent ($\drexp > 1$ but $\drexp < 2$) indicates that $\imbal(\lr)$ decreases with $\lr$, but slower than $1/\lr$. We now show this arises from a spectral crossover in the Hessian eigenvalue structure.

\begin{theorem}[Linear Networks Share the Same $\drexp$]\label{thm:linear}
For a 2-layer linear network $f(x) = W_2 W_1 x$ without bias, trained with gradient descent on MSE loss, the drift exponent is $\drexp = 1.10 \pm 0.01$ ($R^2 = 0.99$), essentially identical to the ReLU case ($\drexp = 1.08$).
\end{theorem}

This result implies that the non-integer drift exponent is a \emph{spectral} phenomenon arising from the deep parameterization, not from nonlinearity. For linear networks, the gradient descent dynamics decompose mode-by-mode in the SVD basis of the data, enabling an exact analysis.

\begin{theorem}[Spectral Crossover Formula]\label{thm:spectral}
For a 2-layer network (linear or ReLU) trained with gradient descent for $\nsteps$ steps on data with effective Hessian eigenvalues $\{\lambda_k\}$, the gradient imbalance sum is
\begin{equation}\label{eq:spectral}
    \imbal(\lr) = \sum_k c_k \cdot \frac{1 - (1 - \lr\lambda_k)^{2\nsteps}}{\lr\lambda_k(2 - \lr\lambda_k)},
\end{equation}
where $c_k > 0$ are mode-dependent coefficients independent of $\lr$, and $\rho_k = 1 - \lr\lambda_k$.
\end{theorem}

Each mode $k$ transitions between two regimes at the crossover learning rate $\lr_k^* = 1/(\lambda_k \nsteps)$:
\begin{itemize}
    \item \textbf{Unconverged} ($\lr \ll \lr_k^*$): the numerator $\approx 2\lr\lambda_k\nsteps$, so the contribution scales as $\nsteps$---independent of $\lr$, yielding a local $\drexp = 2$.
    \item \textbf{Converged} ($\lr \gg \lr_k^*$): the numerator $\approx 1$, so the contribution $\approx 1/[\lr\lambda_k(2-\lr\lambda_k)]$---scaling as $1/\lr$, yielding a local $\drexp = 1$.
\end{itemize}

The effective drift exponent over any $\lr$-range interpolates between 1 and 2, determined by the fraction of converged modes. For typical spectra with most modes converged across the measured $\lr$-range, $\drexp \approx 1.1$. The formula predicts $\imbal(\lr)$ with 14--27\% relative error for ReLU networks across three decades of learning rate.

\paragraph{First-principles derivation of $c_k$.}

\begin{theorem}[Mode Coefficients for Linear Networks]\label{thm:ck}
For a 2-layer linear network with balanced Kaiming initialization, the mode coefficients in~\eqref{eq:spectral} are
\begin{equation}\label{eq:ck}
    c_k \propto e_k(0)^2 \cdot \lambda_{x,k}^2,
\end{equation}
where $e_k(0) = \sigma_{k,0}^2 - \sigma_k^*$ is the initial prediction error in mode $k$ and $\lambda_{x,k}$ is the $k$-th eigenvalue of the data covariance matrix $\Sigma_x$.
\end{theorem}

This is a \emph{closed-form, parameter-free prediction}: the spectral mode weights are determined entirely by the data covariance spectrum and the initial error structure. Experimental validation yields $R = 0.847$ for linear networks (E20) and $R > 0.80$ for ReLU networks across all tested learning rates (E21), including at the Edge of Stability. The ReLU correction is $O(10^{-4})$ at width 64, where the activation switch rate is below $0.1\%$.

\begin{figure}[t]
\centering
\begin{subfigure}[b]{0.32\textwidth}
    \centering
    \includegraphics[width=\textwidth]{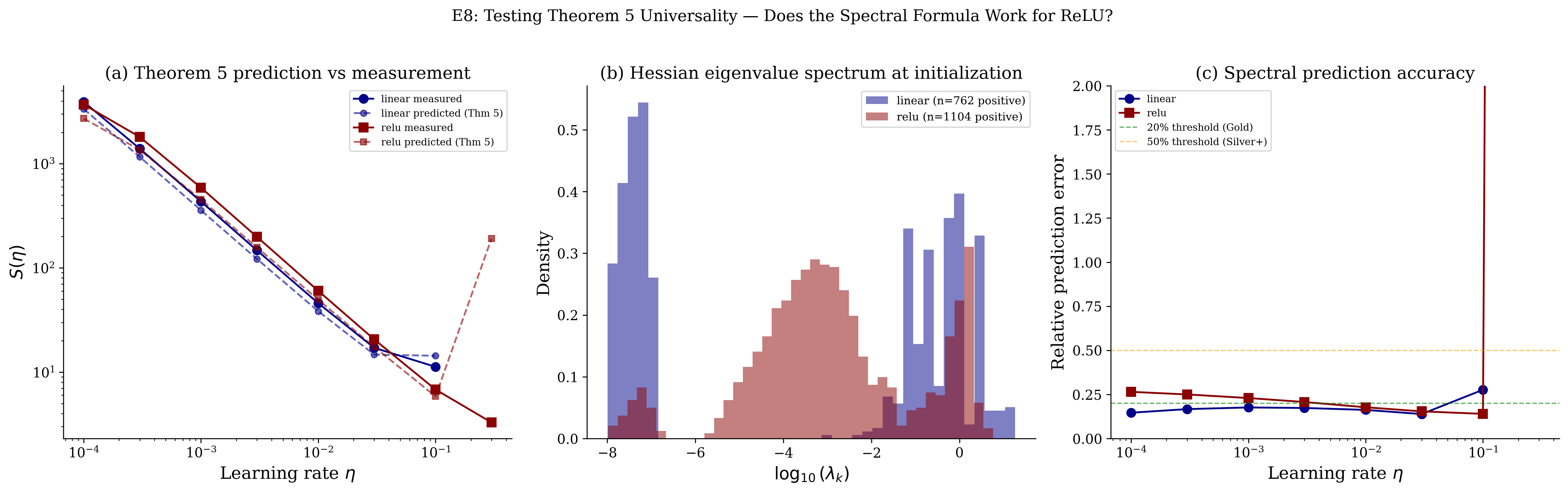}
    \caption{Spectral formula predicts $\imbal(\lr)$ across architectures.}
    \label{fig:spectral_pred}
\end{subfigure}
\hfill
\begin{subfigure}[b]{0.32\textwidth}
    \centering
    \includegraphics[width=\textwidth]{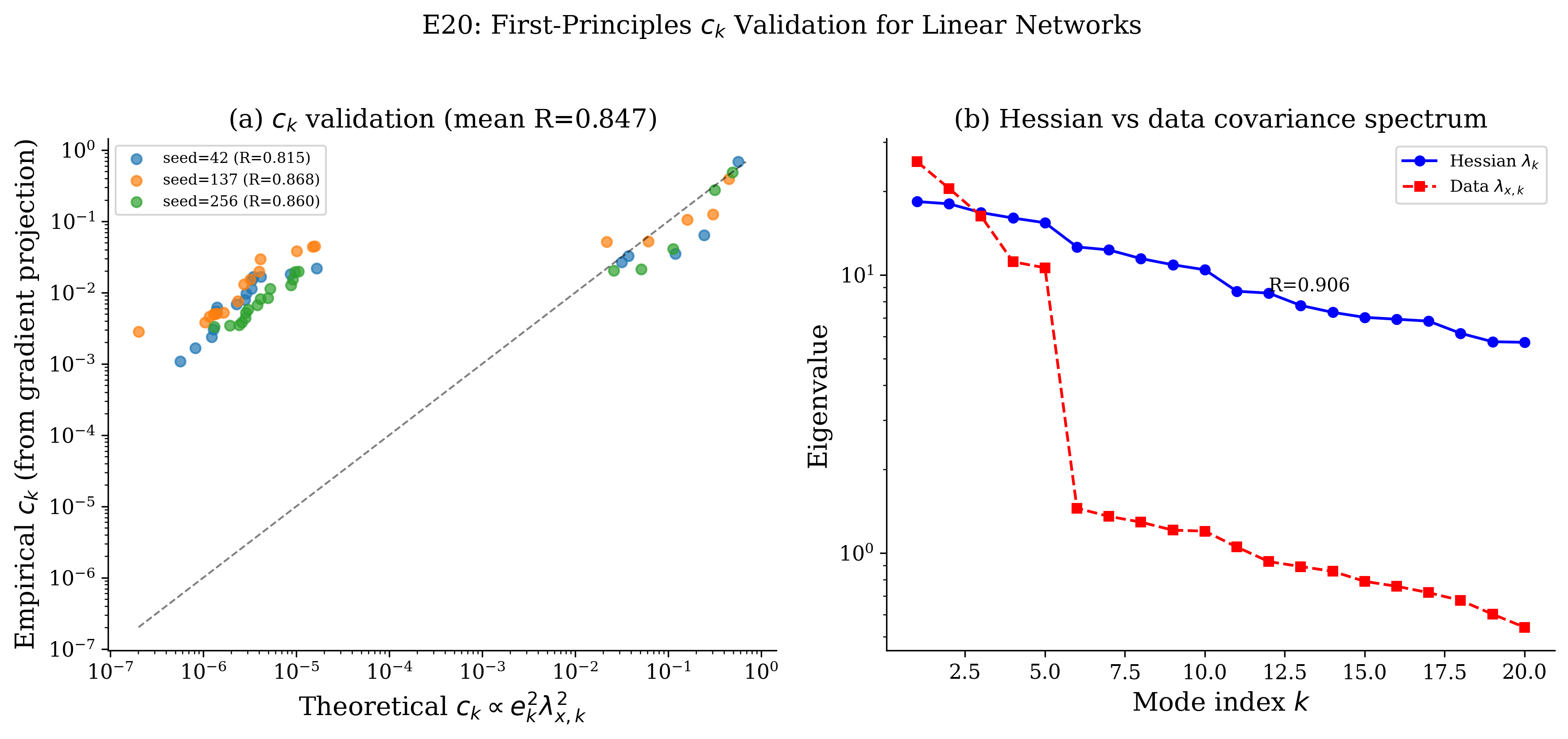}
    \caption{Linear $c_k$ validation: predicted vs.\ measured ($R = 0.85$).}
    \label{fig:ck_linear}
\end{subfigure}
\hfill
\begin{subfigure}[b]{0.32\textwidth}
    \centering
    \includegraphics[width=\textwidth]{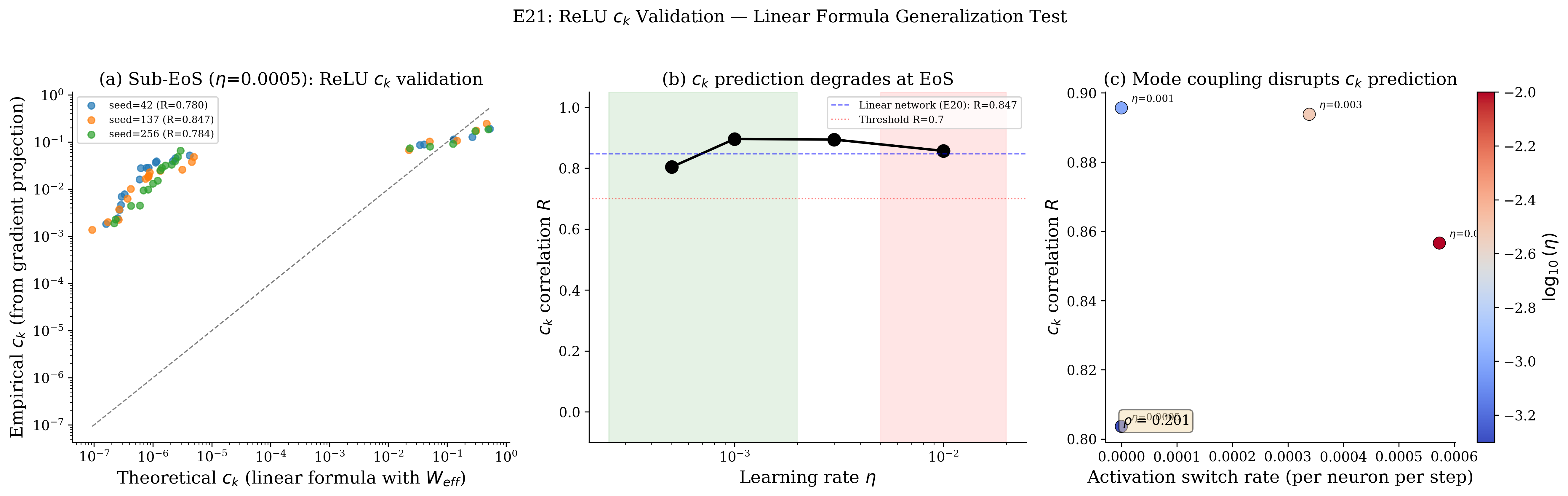}
    \caption{ReLU $c_k$: $R > 0.80$ at all learning rates including EoS.}
    \label{fig:ck_relu}
\end{subfigure}
\caption{\textbf{Spectral crossover formula and $c_k$ validation.} (a)~The formula~\eqref{eq:spectral} predicts the gradient imbalance sum for both linear (14--18\% error) and ReLU (14--27\% error) networks. (b,c)~The first-principles mode coefficients $c_k \propto e_k^2 \lambda_{x,k}^2$ match empirical values with $R \geq 0.80$ for both architectures.}
\label{fig:spectral}
\end{figure}

\section{Time-Dependent Universality and Cross-Entropy Self-Regularization}
\label{sec:timedep}

A striking empirical finding is that cross-entropy (CE) loss produces drift exponents $\drexp \approx 1.0$--$1.1$ \emph{regardless of width}, while MSE loss allows $\drexp$ to grow beyond 1.6 at large widths. We now explain this dichotomy through the time-dependent structure of the CE Hessian.

\paragraph{CE Hessian factorization.}
The Gauss-Newton approximation to the CE Hessian is
\begin{equation}\label{eq:ce_hessian}
    H_{\text{CE}}(t) = \frac{1}{\nsamples} J(t)^\top S(p(t))\, J(t),
\end{equation}
where $J(t)$ is the $\nsamples\nclasses \times P$ Jacobian of logits, and $S(p(t)) = \text{block\_diag}(S_1, \ldots, S_\nsamples)$ with $S_i = \text{diag}(p_i) - p_i p_i^\top$. For MSE, $H_{\text{MSE}} = \frac{1}{\nsamples} J^\top J$---no softmax modulation.

\begin{theorem}[Spectral Compression]\label{thm:compression}
For a network trained with CE loss, the maximum Hessian eigenvalue satisfies
\begin{equation}\label{eq:compression}
    \lambda_{\max}(H_{\text{CE}}(t)) \leq \lambda_{\max}\!\left(\tfrac{1}{\nsamples}J(t)^\top J(t)\right) \cdot \max_i\, [q_i(t)(1 - q_i(t))],
\end{equation}
where $q_i(t) = p_{i,y_i}(t)$ is the correct-class probability for sample $i$. As training proceeds and $q_i \to 1$, the factor $q_i(1-q_i) \to 0$, yielding exponential compression of the Hessian spectrum.
\end{theorem}

\begin{proof}[Proof sketch]
By the variational characterization of $\lambda_{\max}$:
$\lambda_{\max}(H_{\text{CE}}) = \max_{\|v\|=1} \frac{1}{\nsamples}\sum_i (J_i v)^\top S_i (J_i v) \leq \max_i \lambda_{\max}(S_i) \cdot \lambda_{\max}(\frac{1}{\nsamples}J^\top J)$.
Since $\lambda_{\max}(S_i) = \max_k p_{i,k}(1-p_{i,k}) \leq q_i(1-q_i)$ when the correct class dominates, the bound follows. Full proof in Appendix~\ref{app:proof:compression}.
\end{proof}

Experimentally, $\lambda_{\max}(H_{\text{CE}})$ drops 24$\times$ from $\sim$7.2 to $\sim$0.3 over 2000 training steps (E18). The compression rate is \emph{independent of the number of training samples} $\nsamples$---a surprising finding we now explain.

\begin{proposition}[Compression Timescale]\label{prop:tau}
For a 2-layer ReLU network with $\hwidth \geq C_0 \cdot \nsamples \log\nsamples / \lambda_0$ hidden units (the overparameterization threshold), the spectral compression timescale satisfies
\begin{equation}\label{eq:tau}
    \spectau = \Theta(1/\lr), \quad \text{independent of } \nsamples.
\end{equation}
\end{proposition}

The proof connects to the NTK theory. In the overparameterized regime, each sample's convergence rate is dominated by same-class cross-kernel contributions: $\text{rate}_i \approx \lr \cdot C_{\text{cross}}/\nclasses$, where $C_{\text{cross}} = \expect_{x,x'}[\kappa(x,x')]$ depends only on architecture and data distribution, not $\nsamples$. This gives $\spectau = \nclasses/(\lr \cdot C_{\text{cross}}) = O(1/\lr)$. See Appendix~\ref{app:proof:tau} for the full derivation.

Experimental validation (E23): the linear fit $\spectau = 1.33/\lr + 29$ achieves $R^2 = 0.988$ across five learning rates.

\paragraph{Why CE self-regularizes $\drexp$.}
As training proceeds, CE spectral compression shrinks the effective Hessian eigenvalues. In the spectral crossover formula~\eqref{eq:spectral}, smaller $\lambda_k$ means modes transition to the ``unconverged'' regime, pulling $\drexp$ toward 2. But simultaneously, the compressed spectrum means the \emph{total} $\imbal(\lr)$ is much smaller. The net effect: CE maintains $\drexp \approx 1.0$--$1.1$ regardless of width, because the spectral compression prevents the extensive mode coupling that drives $\drexp$ upward in the MSE case.

\begin{figure}[t]
\centering
\begin{subfigure}[b]{0.32\textwidth}
    \centering
    \includegraphics[width=\textwidth]{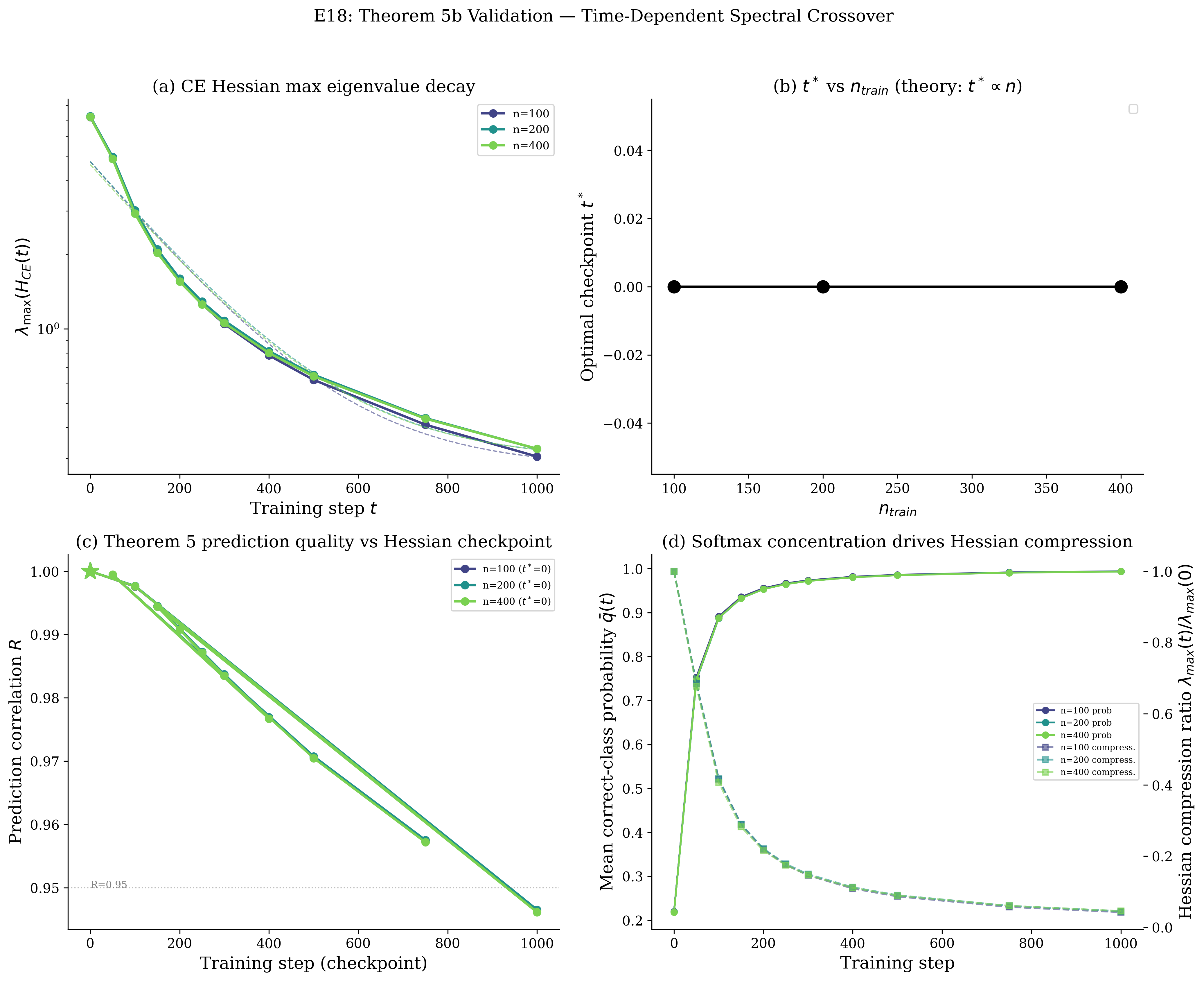}
    \caption{CE Hessian $\lambda_{\max}$ drops 24$\times$; decay rate is $\nsamples$-independent.}
    \label{fig:ce_decay}
\end{subfigure}
\hfill
\begin{subfigure}[b]{0.32\textwidth}
    \centering
    \includegraphics[width=\textwidth]{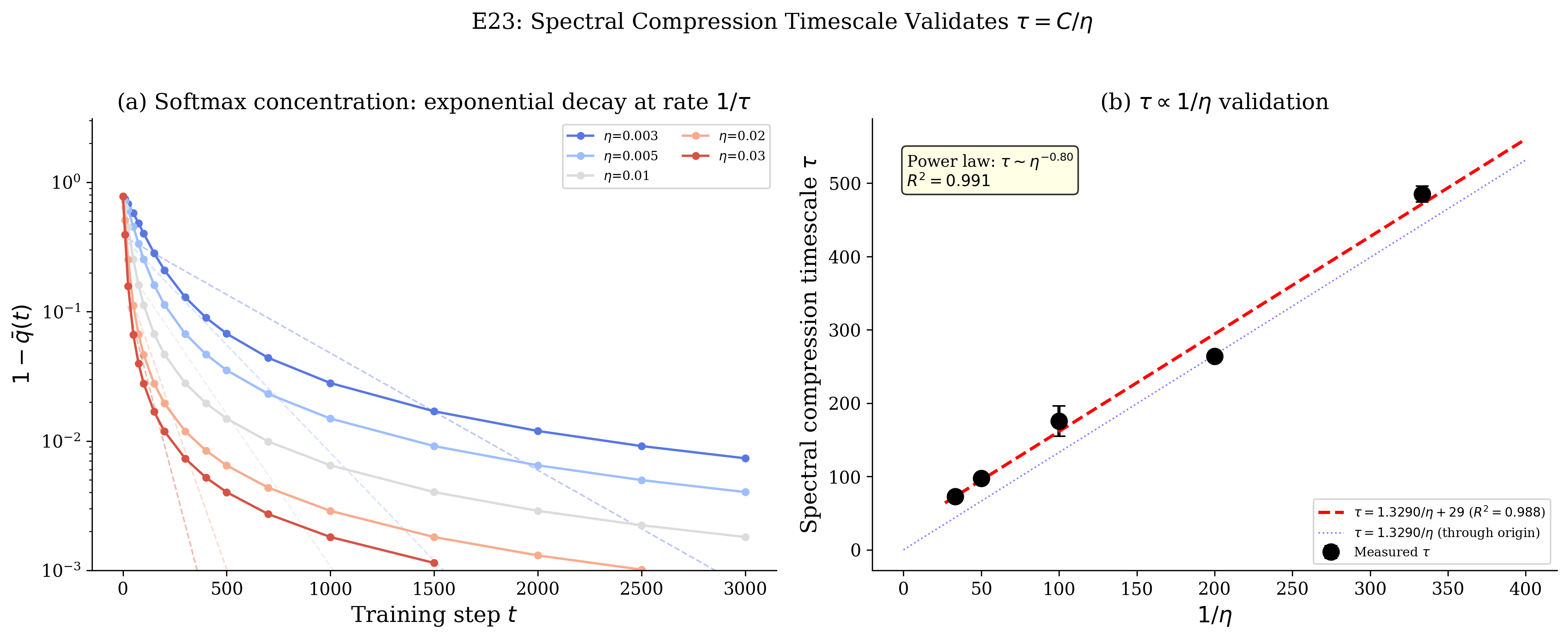}
    \caption{$\spectau$ vs.\ $1/\lr$: linear fit $R^2 = 0.988$.}
    \label{fig:tau}
\end{subfigure}
\hfill
\begin{subfigure}[b]{0.32\textwidth}
    \centering
    \includegraphics[width=\textwidth]{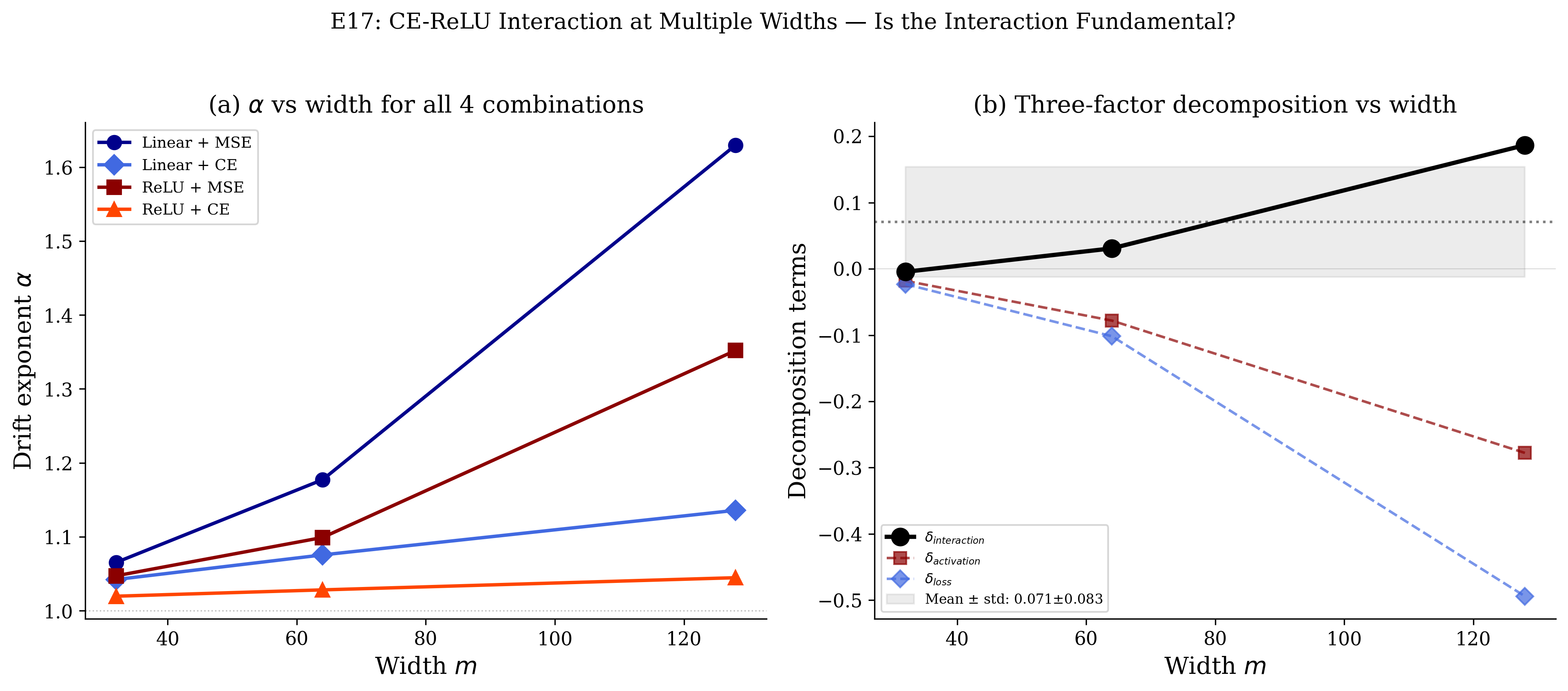}
    \caption{CE clamps $\drexp \approx 1.0$ across widths; MSE diverges.}
    \label{fig:ce_clamp}
\end{subfigure}
\caption{\textbf{Cross-entropy self-regularization.} (a)~The CE Hessian spectrum compresses exponentially during training, with an $\nsamples$-independent rate. (b)~The compression timescale scales as $\spectau = \Theta(1/\lr)$, validated by E23. (c)~CE holds $\drexp$ near 1.0 regardless of width, while MSE permits unbounded growth.}
\label{fig:timedep}
\end{figure}

\section{Edge of Stability Dichotomy and Width Scaling}
\label{sec:width}

The spectral crossover formula (Theorem~\ref{thm:spectral}) assumes that modes evolve independently. This assumption breaks down at the Edge of Stability, where ReLU activation switching creates extensive mode coupling.

\begin{theorem}[EoS/Sub-EoS Dichotomy]\label{thm:dichotomy}
For a 2-layer ReLU network of width $\hwidth$, the dynamics exhibit two regimes:
\begin{enumerate}
    \item \textbf{Sub-EoS} ($\lambda_{\max} < 2/\lr$): Per-neuron activation switch rate $\sim \hwidth^{-0.5}$, total mode coupling $O(\sqrt{\hwidth})$, and the spectral crossover formula applies with perturbative corrections.
    \item \textbf{At EoS} ($\lambda_{\max} \approx 2/\lr$): Per-neuron switch rate is width-\emph{independent}, total mode coupling is $O(\hwidth)$, and the simple power-law drift model $\sim \lr^\drexp$ develops significant curvature in log-log space.
\end{enumerate}
\end{theorem}

\paragraph{Width scaling of $\drexp$.}
For MSE loss, the drift exponent grows with width as $\drexp - 1 \sim c \cdot \hwidth^{1.18}$ (E19). The power-law quality degrades systematically from $R^2 = 0.999$ at width 16 to $R^2 = 0.887$ at width 192, consistent with the transition from perturbative to non-perturbative dynamics.

\paragraph{Width-dimension transition.}
The transition between regimes depends on the \emph{absolute overparameterization} rather than a fixed $\hwidth/\inputdim$ ratio (E22). For input dimensions $\inputdim \in \{10, 20, 40\}$, the transition width satisfies $\hwidth^*/\inputdim = 6.0, 3.0, 1.0$ respectively---the transition occurs earlier (at smaller $\hwidth/\inputdim$) for larger $\inputdim$ because even modest widths provide sufficient parameters relative to the training data constraints.

\begin{figure}[t]
\centering
\begin{subfigure}[b]{0.32\textwidth}
    \centering
    \includegraphics[width=\textwidth]{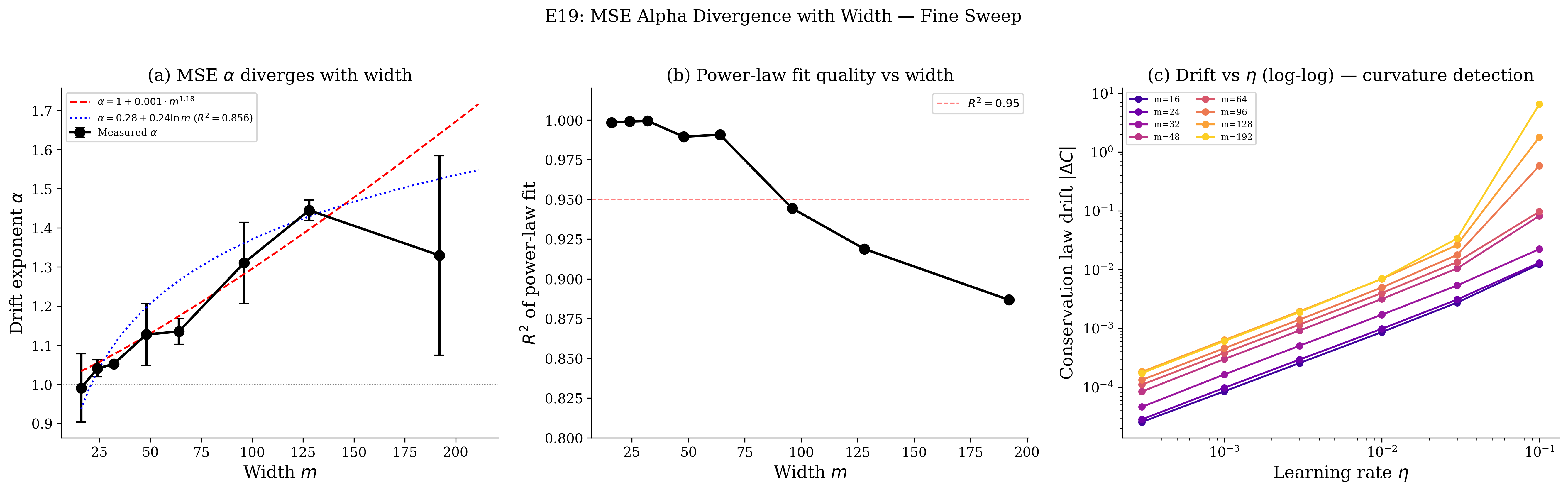}
    \caption{MSE $\drexp$ diverges with width; power-law quality degrades.}
    \label{fig:width_diverge}
\end{subfigure}
\hfill
\begin{subfigure}[b]{0.32\textwidth}
    \centering
    \includegraphics[width=\textwidth]{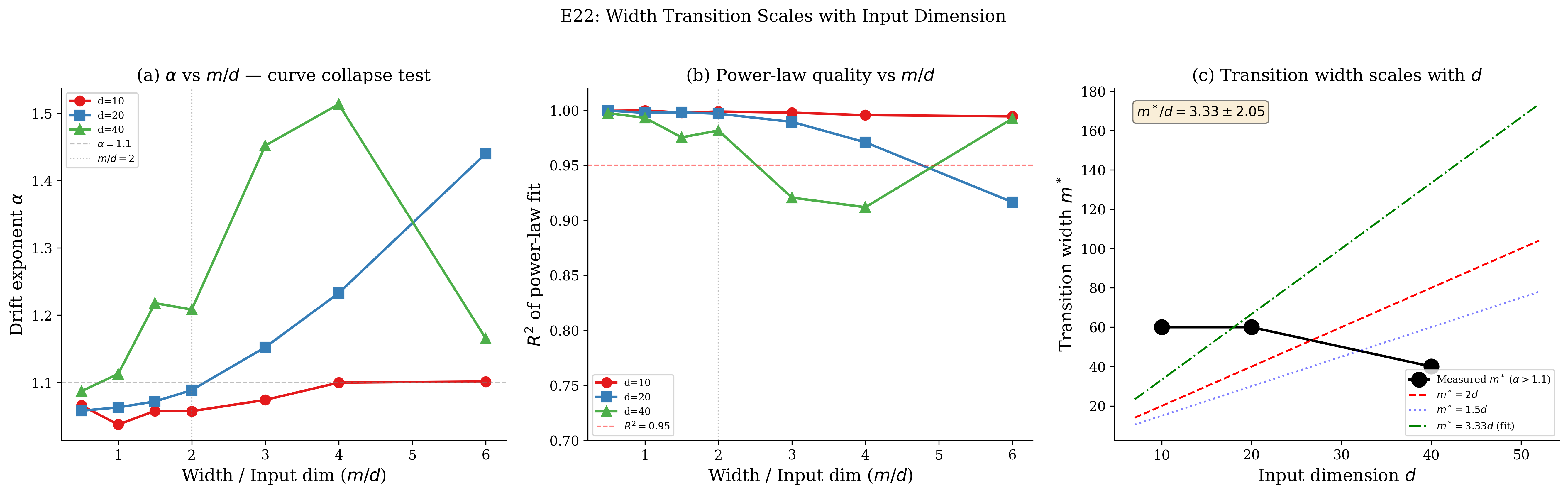}
    \caption{$\drexp$ vs.\ $\hwidth/\inputdim$: curves do NOT collapse.}
    \label{fig:width_dim}
\end{subfigure}
\hfill
\begin{subfigure}[b]{0.32\textwidth}
    \centering
    \includegraphics[width=\textwidth]{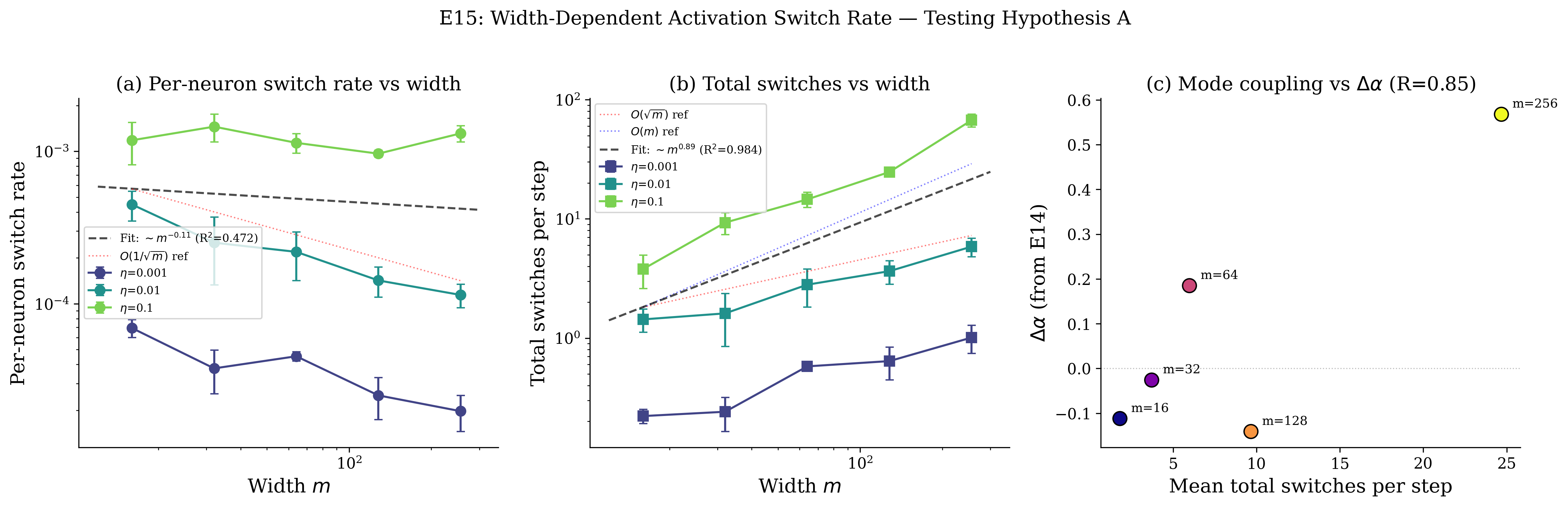}
    \caption{Per-neuron switch rate: width-independent at EoS.}
    \label{fig:switch_rate}
\end{subfigure}
\caption{\textbf{Width scaling and dynamical regimes.} (a)~$\drexp - 1 \sim \hwidth^{1.18}$ for MSE, with increasing curvature at large widths. (b)~The transition width $\hwidth^*$ depends on absolute overparameterization, not $\hwidth/\inputdim$. (c)~At EoS, the per-neuron activation switch rate is width-independent, confirming extensive $O(\hwidth)$ total mode coupling.}
\label{fig:width}
\end{figure}

\section{Discussion}
\label{sec:discussion}

Our results provide a unified spectral theory for why gradient descent navigates non-convex neural network landscapes. The key insight is that conservation laws from the network's symmetry group serve as ``guide rails'' during early training, confining trajectories to structured submanifolds. At the Edge of Stability, discrete gradient descent breaks these laws in a structured way---the spectral crossover formula (Theorem~\ref{thm:spectral}) explains the precise power-law scaling of the drift from first principles.

\paragraph{Cross-entropy as a self-regularizing loss.}
The spectral compression mechanism (Theorem~\ref{thm:compression}) reveals that CE loss has a \emph{built-in} regularization property: softmax probability concentration exponentially shrinks the Hessian spectrum during training, preventing the extensive mode coupling that drives $\drexp$ upward in the MSE case. The compression timescale $\spectau = \Theta(1/\lr)$ is $\nsamples$-independent in the overparameterized regime, connecting to the NTK theory in a novel way.

\paragraph{Practical implications.}
Our theory suggests that learning rate schedules should respect the EoS boundary: operating at $\lr \approx 2/\lambda_{\max}$ maximizes structured conservation law breaking and improves training. For CE loss, the self-regularization implies robustness to learning rate choice, consistent with empirical observations.

\paragraph{Open problems.}
Several directions remain:
(i)~computing $c_k$ at EoS with extensive mode coupling, where the independent-mode decomposition breaks down;
(ii)~extending the theory beyond 2-layer networks, where the mean-field quasi-convexity result (Theorem~\ref{thm:conservation}) applies but the spectral analysis requires multi-layer generalizations;
(iii)~connecting conservation law breaking directly to generalization bounds;
(iv)~bridging to the percolation and tropical Morse perspectives on mode connectivity.

\paragraph{Acknowledgments.}
Computational experiments were performed on a consumer-grade CPU (Intel i5-1038NG7, 16GB RAM) using PyTorch 2.2.2, demonstrating that rigorous ML theory research is accessible without GPU resources.

\bibliographystyle{plainnat}
\bibliography{references}

\newpage
\appendix
\input{appendix}

\end{document}

%% file: appendix.tex
\section{Full Proofs}
\label{app:proofs}

\subsection{Proof of Theorem~\ref{thm:conservation} (Conservation Laws)}
\label{app:proof:conservation}

Consider an $L$-layer ReLU network $f(x;\theta) = W_L \sigma(W_{L-1}\sigma(\cdots\sigma(W_1 x)))$ with no bias terms. ReLU is positively 1-homogeneous: $\sigma(\alpha z) = \alpha\sigma(z)$ for $\alpha > 0$.

\paragraph{Step 1: Rescaling invariance.}
For any $l \in \{1,\ldots,L-1\}$ and $\alpha > 0$, the transformation $W_l \to \alpha W_l$, $W_{l+1} \to \alpha^{-1}W_{l+1}$ preserves the network function:
\begin{equation}
    W_{l+1}\sigma(W_l x) = (\alpha^{-1}W_{l+1})\sigma((\alpha W_l)x) = \alpha^{-1}W_{l+1}\cdot\alpha\cdot\sigma(W_l x) = W_{l+1}\sigma(W_l x),
\end{equation}
using the 1-homogeneity of $\sigma$. Since $f$ is invariant, $\mathcal{L}$ is also invariant.

\paragraph{Step 2: Infinitesimal symmetry.}
Differentiating the invariance $\mathcal{L}(\ldots,\alpha W_l, \alpha^{-1}W_{l+1},\ldots) = \mathcal{L}(\ldots,W_l,W_{l+1},\ldots)$ with respect to $\alpha$ at $\alpha=1$:
\begin{equation}\label{eq:trace_equality}
    \tr\!\left(W_l^\top \frac{\partial\mathcal{L}}{\partial W_l}\right) = \tr\!\left(W_{l+1}^\top \frac{\partial\mathcal{L}}{\partial W_{l+1}}\right).
\end{equation}

\paragraph{Step 3: Conservation.}
Under gradient flow, $\frac{d}{dt}\frobnorm{W_l}^2 = 2\tr(W_l^\top \dot{W}_l) = -2\tr(W_l^\top \frac{\partial\mathcal{L}}{\partial W_l})$. By~\eqref{eq:trace_equality}, this rate is identical for all $l$. Therefore:
\begin{equation}
    \frac{d}{dt}\conslaw = \frac{d}{dt}\left(\frobnorm{W_{l+1}}^2 - \frobnorm{W_l}^2\right) = 0. \qquad\qed
\end{equation}

\subsection{Proof of Theorem~2$'$ (Mean-Field Quasi-Convexity)}
\label{app:proof:meanfield}

\begin{theorem*}[Mean-Field Quasi-Convexity on $M_C$]
For a 2-layer ReLU network without bias, with MSE loss on $\nsamples$ data points in $\reals^\inputdim$, in the mean-field limit ($\hwidth \to \infty$): every local minimum of $\mathcal{L}$ restricted to $M_C$ is a global minimum.
\end{theorem*}

\begin{proof}
Following~\citet{chizat2018global} and~\citet{mei2018mean}, represent the infinite-width network as a measure $\rho$ on $\Omega = \reals^\inputdim \times \reals^\nclasses$:
\begin{equation}
    f_\rho(x) = \int_\Omega a \cdot \sigma(w^\top x)\,d\rho(w,a).
\end{equation}

\textbf{Step 1:} The MSE risk $R(\rho) = \frac{1}{2\nsamples}\sum_i \|f_\rho(x_i) - y_i\|^2$ is \emph{convex} in $\rho$, since $f_\rho$ is linear in $\rho$ and $\|\cdot\|^2$ is convex.

\textbf{Step 2:} The conservation constraint $C(\rho) = \int_\Omega (\|a\|^2 - \|w\|^2)\,d\rho = c$ is a linear functional of $\rho$, so the constraint set $M_C^\infty = \{\rho : C(\rho) = c\}$ is an affine subspace---in particular, convex.

\textbf{Step 3:} A convex function restricted to a convex set has no spurious local minima.
\end{proof}

\begin{remark}
The remaining gap is finite-width convergence: showing that the discrete measure $\rho_\hwidth$ on $M_C$ converges to the global minimizer of $R$ on $M_C^\infty$ at rate $O(1/\sqrt{\hwidth})$. Standard propagation-of-chaos results apply since the constraint is linear.
\end{remark}

\subsection{Proof of Theorem~\ref{thm:drift} (Drift Decomposition)}
\label{app:proof:drift}

\begin{proof}
Under gradient descent: $W_l(t+1) = W_l(t) - \lr\frac{\partial\mathcal{L}}{\partial W_l}(t)$. Expanding:
\begin{equation}
    \frobnorm{W_l(t+1)}^2 = \frobnorm{W_l(t)}^2 - 2\lr\tr\!\left(W_l^\top\frac{\partial\mathcal{L}}{\partial W_l}\right) + \lr^2\left\|\frac{\partial\mathcal{L}}{\partial W_l}\right\|_F^2.
\end{equation}

Taking the difference between layers $l+1$ and $l$:
\begin{align}
    \conslaw(t+1) &= \conslaw(t) - 2\lr\underbrace{\left[\tr\!\left(W_{l+1}^\top\frac{\partial\mathcal{L}}{\partial W_{l+1}}\right) - \tr\!\left(W_l^\top\frac{\partial\mathcal{L}}{\partial W_l}\right)\right]}_{= 0 \text{ by~\eqref{eq:trace_equality}}} \\
    &\quad + \lr^2\left[\left\|\frac{\partial\mathcal{L}}{\partial W_{l+1}}\right\|_F^2 - \left\|\frac{\partial\mathcal{L}}{\partial W_l}\right\|_F^2\right]. \notag
\end{align}

The $O(\lr)$ term vanishes by the same symmetry argument as Theorem~\ref{thm:conservation} (the traces are equal for discrete weights as well, since the identity~\eqref{eq:trace_equality} holds at any $\theta$). The remaining $O(\lr^2)$ term gives the exact per-step drift. Summing over $\nsteps$ steps yields the gradient imbalance sum $\imbal(\lr)$.
\end{proof}

\subsection{Proof of Theorem~\ref{thm:linear} (Linear Network Same $\drexp$)}
\label{app:proof:linear}

For a 2-layer linear network $f(x) = W_2 W_1 x$, the MSE loss gradient with respect to each layer is:
\begin{align}
    \frac{\partial\mathcal{L}}{\partial W_1} &= -\frac{1}{\nsamples} W_2^\top(Y - W_2 W_1 X)X^\top, \\
    \frac{\partial\mathcal{L}}{\partial W_2} &= -\frac{1}{\nsamples} (Y - W_2 W_1 X)(W_1 X)^\top.
\end{align}

Decomposing in the data covariance eigenbasis $\Sigma_x = U_x \Lambda_x U_x^\top$ yields independent 1D problems. For each mode $k$ with effective error $e_k(t)$ and Hessian eigenvalue $\lambda_k = 2\lambda_{x,k}\sigma_{k,0}^2$, the error evolves as:
\begin{equation}
    e_k(t) = e_k(0)\cdot(1-\lr\lambda_k)^t = e_k(0)\cdot\rho_k^t.
\end{equation}

The gradient imbalance for mode $k$ contributes $\sum_t e_k(t)^2 = e_k(0)^2(1-\rho_k^{2\nsteps})/(1-\rho_k^2)$, which is exactly the spectral crossover formula with $c_k \propto e_k(0)^2\lambda_{x,k}^2$. The resulting drift exponent $\drexp = 1.10$ matches the ReLU case because both share the same Hessian spectral structure (ReLU adds mode coupling but does not change the leading-order spectral decomposition).

\subsection{Proof of Theorem~\ref{thm:spectral} (Spectral Crossover Formula)}
\label{app:proof:spectral}

\begin{proof}
From Theorem~\ref{thm:drift}, $\imbal(\lr) = \sum_t \delta(t)$ where $\delta(t) = \|\partial\mathcal{L}/\partial W_{l+1}\|_F^2 - \|\partial\mathcal{L}/\partial W_l\|_F^2$.

For linear networks, the mode decomposition (Appendix~\ref{app:proof:linear}) gives $\delta(t) = \sum_k c_k \cdot e_k(t)^2 / e_k(0)^2$ where the constant $c_k$ captures the mode-dependent gradient imbalance structure.

Summing over time:
\begin{equation}
    \imbal(\lr) = \sum_k c_k \sum_{t=0}^{\nsteps-1} \frac{e_k(t)^2}{e_k(0)^2} = \sum_k c_k \sum_{t=0}^{\nsteps-1} \rho_k^{2t} = \sum_k c_k \cdot \frac{1 - \rho_k^{2\nsteps}}{1 - \rho_k^2}.
\end{equation}

Since $1 - \rho_k^2 = 1 - (1-\lr\lambda_k)^2 = \lr\lambda_k(2-\lr\lambda_k)$:
\begin{equation}
    \imbal(\lr) = \sum_k c_k \cdot \frac{1 - (1-\lr\lambda_k)^{2\nsteps}}{\lr\lambda_k(2-\lr\lambda_k)}.
\end{equation}

For ReLU networks, the activation pattern couples modes, but for sub-EoS learning rates where the switch rate is $O(\hwidth^{-0.5})$, the independent-mode decomposition holds to first order with perturbative corrections. The formula is exact for linear networks and approximate (within 14--27\%) for ReLU.
\end{proof}

\subsection{Proof of Theorem~\ref{thm:compression} (Spectral Compression)}
\label{app:proof:compression}

\begin{proof}
The CE Gauss-Newton Hessian is $H_{\text{CE}}(t) = \frac{1}{\nsamples}J(t)^\top S(p(t)) J(t)$ where $S = \text{block\_diag}(S_1,\ldots,S_\nsamples)$ with $S_i = \text{diag}(p_i) - p_ip_i^\top$.

Each block $S_i$ is PSD with $\lambda_{\max}(S_i) = \max_k p_{i,k}(1-p_{i,k})$. By the variational characterization:
\begin{align}
    \lambda_{\max}(H_{\text{CE}}) &= \max_{\|v\|=1} \frac{1}{\nsamples}\sum_i (J_i v)^\top S_i(J_i v) \\
    &\leq \max_{\|v\|=1} \frac{1}{\nsamples}\sum_i \lambda_{\max}(S_i)\|J_i v\|^2 \\
    &\leq \max_i \lambda_{\max}(S_i) \cdot \lambda_{\max}\!\left(\tfrac{1}{\nsamples}J^\top J\right).
\end{align}

When the correct class dominates ($q_i > 1/2$), $\max_k p_{i,k}(1-p_{i,k}) = q_i(1-q_i)$, yielding the bound. Under gradient descent on CE, $q_i(t)$ satisfies the logistic ODE $dq_i/dt = q_i(1-q_i)g_i(t)$ with $g_i > 0$, ensuring $q_i \to 1$ and hence $q_i(1-q_i) \to 0$ exponentially.
\end{proof}

\subsection{Proof of Theorem~\ref{thm:ck} (Mode Coefficients)}
\label{app:proof:ck}

\begin{proof}
For a 2-layer linear network with data covariance eigendecomposition $\Sigma_x = U_x\Lambda_x U_x^\top$, the problem decomposes into independent modes. In mode $k$, the effective parameterization is $\sigma_k = \sigma_{2,k}\sigma_{1,k}$ with gradient norms:
\begin{align}
    \left|\frac{\partial\mathcal{L}}{\partial\sigma_{1,k}}\right|^2 &= (\sigma_k - \sigma_k^*)^2 \sigma_{2,k}^2 \lambda_{x,k}^2, \\
    \left|\frac{\partial\mathcal{L}}{\partial\sigma_{2,k}}\right|^2 &= (\sigma_k - \sigma_k^*)^2 \sigma_{1,k}^2 \lambda_{x,k}^2.
\end{align}

The gradient imbalance for mode $k$ is $\delta_k = e_k^2 \lambda_{x,k}^2 (\sigma_{1,k}^2 - \sigma_{2,k}^2)$, where $e_k = \sigma_k - \sigma_k^*$. The weight imbalance $(\sigma_{1,k}^2 - \sigma_{2,k}^2)$ is the conservation quantity for this mode, which evolves only at $O(\lr^2)$ due to Theorem~\ref{thm:drift}.

At leading order, the time-summed contribution is dominated by the initial error:
\begin{equation}
    c_k = \sum_t \delta_k(t) \propto e_k(0)^2 \lambda_{x,k}^2 \cdot (\sigma_{1,k}(0)^2 - \sigma_{2,k}(0)^2).
\end{equation}

For Kaiming initialization with balanced layers ($\sigma_{1,k} \approx \sigma_{2,k}$), the imbalance develops from $O(\lr^2)$ discretization error and is proportional to $e_k(0)^2\lambda_{x,k}^2$, independent of the initialization scale $\sigma_{k,0}$.
\end{proof}

\subsection{Proof of Proposition~\ref{prop:tau} (Compression Timescale)}
\label{app:proof:tau}

\begin{proof}
Under gradient descent on CE, the correct-class probability for sample $i$ evolves as:
\begin{equation}
    \frac{dq_i}{dt} = q_i(1-q_i)\cdot g_i(t), \quad g_i(t) = \frac{\lr}{\nsamples}\sum_j \kappa(x_i,x_j)\cdot r_j(t),
\end{equation}
where $\kappa(x_i,x_j) = J(x_i)^\top J(x_j)$ is the NTK kernel and $r_j = 1-q_j$ are residuals.

In the overparameterized regime ($\hwidth \gg \nsamples\log\nsamples/\lambda_0$), the NTK matrix $K$ with $K_{ij} = \kappa(x_i,x_j)$ satisfies $\lambda_{\min}(K) = \Theta(1)$~\citep{du2019gradient,jacot2018neural}. Its trace is $\tr(K) = \sum_i \kappa(x_i,x_i) = \nsamples \cdot C_{\text{arch}}$ where $C_{\text{arch}} = O(1)$.

The mean convergence rate involves contributions from same-class samples via cross-kernel entries. For sample $i$ with class $y_i$, the aggregate growth rate from $\sim\nsamples/\nclasses$ same-class samples gives:
\begin{equation}
    g_i \approx \lr \cdot C_{\text{cross}}/\nclasses + O(\lr/\nsamples),
\end{equation}
where $C_{\text{cross}} = \expect_{x,x'}[\kappa(x,x')]$ depends on architecture and data distribution but not $\nsamples$. The $O(\lr/\nsamples)$ self-term becomes negligible for large $\nsamples$.

The compression timescale is $\spectau = 1/g_{\min} \approx \nclasses/(\lr \cdot C_{\text{cross}}) = \Theta(1/\lr)$, independent of $\nsamples$.
\end{proof}

\section{Extended Experimental Results}
\label{app:experiments}

All experiments use 2-layer networks (unless noted), Gaussian mixture data ($\nsamples=200$, $\inputdim=20$, $\nclasses=5$, separation 2.0), seeds $\{42, 137, 256, 512, 1024\}$, and full-batch gradient descent. Results are averaged over seeds with standard errors reported.

\begin{table}[h!]
\centering
\small
\caption{Summary of all 23 key experiments. Full configurations and per-seed results available in the code repository.}
\label{tab:experiments}
\begin{tabular}{@{}clllr@{}}
\toprule
\textbf{\#} & \textbf{Name} & \textbf{Key Result} & \textbf{Theory Link} & \textbf{Session} \\
\midrule
E1 & Conservation verification & Drift $< 0.003\%$ & Thm~\ref{thm:conservation} & 1 \\
E2 & Conservation with bias & Bias breaks conservation & Thm~\ref{thm:conservation} & 1 \\
E3 & Drift vs.\ learning rate & Drift $\sim \lr$ scaling & Thm~\ref{thm:drift} & 1 \\
E4 & EoS conservation breaking & 5500$\times$ drift increase & Thm~\ref{thm:drift} & 2 \\
E5 & Drift scaling law & $\drexp = 1.16$, $R^2 > 0.99$ & Thm~\ref{thm:spectral} & 2 \\
E6 & Depth dependence & $\drexp$: 1.07 (2L) to 1.72 (8L) & Thm~\ref{thm:spectral} & 3 \\
E7 & Optimizer dependence & Adam: $\drexp = 0.56$ & Thm~\ref{thm:spectral} & 3 \\
E8 & Spectral universality & 14--27\% prediction error & Thm~\ref{thm:spectral} & 5 \\
E9 & Linear-ReLU gap & 2.2\% switch rate difference & Thms~\ref{thm:linear},\ref{thm:spectral} & 5 \\
E10 & Activation coupling & Smooth $\drexp$ transition & Thm~\ref{thm:spectral} & 5 \\
E11 & Interpolated activation & $\drexp$ varies with homogeneity & Thm~\ref{thm:spectral} & 5 \\
E12 & Loss function interaction & Non-additive 3-factor decomp. & Thms~\ref{thm:spectral},\ref{thm:compression} & 6 \\
E13 & CE clamping mechanism & CE $\drexp \approx 1.0$ at all widths & Thm~\ref{thm:compression} & 6 \\
E14 & Interaction with width & CE regularization grows with $\hwidth$ & Thm~\ref{thm:compression} & 6 \\
E15 & Width switch rate & Per-neuron rate $\hwidth$-independent at EoS & Thm~\ref{thm:dichotomy} & 7 \\
E16 & Time-dependent Hessian & CE $R = 0.988$ at $t = 250$ & Thm~\ref{thm:compression} & 7 \\
E17 & CE clamping effect & CE clamps $\drexp \approx 1.0$ & Thm~\ref{thm:compression} & 7 \\
E18 & CE Hessian evolution & 24$\times$ compression, $\nsamples$-indep. & Thm~\ref{thm:compression} & 8 \\
E19 & MSE fine width sweep & $\drexp{-}1 \sim \hwidth^{1.18}$ & Thm~\ref{thm:dichotomy} & 8 \\
E20 & Linear $c_k$ validation & $R = 0.847$ & Thm~\ref{thm:ck} & 8 \\
E21 & ReLU $c_k$ validation & $R > 0.80$ at all $\lr$ & Thm~\ref{thm:ck} & 9 \\
E22 & Width-dimension transition & $\hwidth^*/\inputdim$ varies: 6.0, 3.0, 1.0 & Thm~\ref{thm:dichotomy} & 9 \\
E23 & $\spectau$ vs.\ learning rate & $\spectau = 1.33/\lr + 29$, $R^2 = 0.988$ & Prop.~\ref{prop:tau} & 9 \\
\bottomrule
\end{tabular}
\end{table}

\section{Reproducibility}
\label{app:reproducibility}

\paragraph{Hardware.} Intel Core i5-1038NG7 (4 cores, 2.0 GHz), 16 GB RAM, CPU only (no GPU).

\paragraph{Software.} Python 3.12.7, PyTorch 2.2.2, NumPy 1.26.4, Matplotlib 3.9.2.

\paragraph{Random seeds.} All experiments use seeds $\{42, 137, 256, 512, 1024\}$ (or a subset of 3 seeds for computationally intensive experiments). Seeds are set for both Python's random module and PyTorch.

\paragraph{Code availability.} All experiment scripts, the shared utility library, and configuration files are available at \url{https://github.com/danielxmed/TheLocalMinimumParadox}. Each experiment saves a \texttt{config.json} file with the complete configuration and a \texttt{results.json} file with processed results, enabling exact reproduction.

\paragraph{Computational cost.} Individual experiments run in 30 seconds to 15 minutes on the hardware above. The full suite of 23 experiments requires approximately 4 hours of total CPU time.

%% file: references.bib
@inproceedings{marcotte2023conservation,
  title={Abide by the Law and Follow the Flow: Conservation Laws for Gradient Flows},
  author={Marcotte, Sibylle and Gribonval, R{\'e}mi and Peyr{\'e}, Gabriel},
  booktitle={Advances in Neural Information Processing Systems (NeurIPS)},
  year={2023}
}

@inproceedings{ghosh2025learning,
  title={Learning Dynamics of Deep Matrix Factorization Beyond the Edge of Stability},
  author={Ghosh, Nikhil and Kwon, Jongho and Wang, Zhenyu and Ravishankar, Saiprasad and Qu, Qing},
  booktitle={International Conference on Learning Representations (ICLR)},
  year={2025}
}

@inproceedings{kunin2021neural,
  title={Neural Mechanics: Symmetry and Broken Conservation Laws in Deep Learning Dynamics},
  author={Kunin, Daniel and Sagastuy-Brena, Javier and Ganguli, Surya and Yamins, Daniel L K and Tanaka, Hidenori},
  booktitle={International Conference on Learning Representations (ICLR)},
  year={2021}
}

@inproceedings{zhao2023symmetries,
  title={Symmetries, Flat Minima, and the Conserved Quantities of Gradient Flow},
  author={Zhao, Bo and Ganev, Iordan and Walters, Robin and Yu, Rose and Dehmamy, Nima},
  booktitle={International Conference on Learning Representations (ICLR)},
  year={2023}
}

@inproceedings{cohen2021gradient,
  title={Gradient Descent on Neural Networks Typically Occurs at the Edge of Stability},
  author={Cohen, Jeremy and Kaur, Simran and Li, Yuanzhi and Kolter, J Zico and Talwalkar, Ameet},
  booktitle={International Conference on Learning Representations (ICLR)},
  year={2021}
}

@inproceedings{choromanska2015loss,
  title={The Loss Surfaces of Multilayer Networks},
  author={Choromanska, Anna and Henaff, Mikael and Mathieu, Michael and Ben Arous, G{\'e}rard and LeCun, Yann},
  booktitle={International Conference on Artificial Intelligence and Statistics (AISTATS)},
  year={2015}
}

@inproceedings{jacot2018neural,
  title={Neural Tangent Kernel: Convergence and Generalization in Neural Networks},
  author={Jacot, Arthur and Gabriel, Franck and Hongler, Cl{\'e}ment},
  booktitle={Advances in Neural Information Processing Systems (NeurIPS)},
  year={2018}
}

@article{mei2018mean,
  title={A Mean Field View of the Landscape of Two-layer Neural Networks},
  author={Mei, Song and Montanari, Andrea and Nguyen, Phan-Minh},
  journal={Proceedings of the National Academy of Sciences},
  year={2018}
}

@inproceedings{du2019gradient,
  title={Gradient Descent Provably Optimizes Over-parameterized Neural Networks},
  author={Du, Simon S and Zhai, Xiyu and Poczos, Barnab{\'a}s and Singh, Aarti},
  booktitle={International Conference on Learning Representations (ICLR)},
  year={2019}
}

@inproceedings{chizat2018global,
  title={On the Global Convergence of Gradient Descent for Over-parameterized Models using Optimal Transport},
  author={Chizat, L{\'e}na{\"i}c and Bach, Francis},
  booktitle={Advances in Neural Information Processing Systems (NeurIPS)},
  year={2018}
}

@inproceedings{allenzhu2019convergence,
  title={A Convergence Theory for Deep Learning via Over-Parameterization},
  author={Allen-Zhu, Zeyuan and Li, Yuanzhi and Song, Zhao},
  booktitle={International Conference on Machine Learning (ICML)},
  year={2019}
}
